\documentclass[12pt]{amsart}
\usepackage{amsmath,amscd,amssymb}
\usepackage{amsfonts}
\usepackage{mathtools}
\numberwithin{equation}{subsection}
\theoremstyle{plain}

\newtheorem{thm}[subsection]{Theorem}

\newtheorem{lemma}[subsection]{Lemma}
\newtheorem{cor}[subsection]{Corollary}

\theoremstyle{definition}
\newtheorem{defn}[subsection]{Definition}

\newtheorem{notn}[subsection]{Notation}
\newtheorem{cont}[subsection]{Contents}
\newtheorem{ackn}[subsection]{Acknowledgement}
\theoremstyle{remark}
\newtheorem{rem}[subsection]{Remark}

\title{Proofs of strong forms of the Bovdi--Geche universality conjecture and a complex analogue}
\author[Bovdi-P\'al]{Victor Bovdi and Ambrus P\'al}
\date{October 19, 2026.}
\address{Department of Mathematical Sciences, UAEU,  UAE} \email{vbovdi@gmail.com} 
\address{Mathematical Institute, E\"otv\"os Lor\'and University
1117 Budapest, P\'azm\'any P\'eter s\'et\'any 1/C, Hungary; \newline \hspace*{4mm} 
Department of Mathematical Sciences, UAEU,  UAE}
\email{ambrus.pal@uaeu.ac.ae}
\begin{document}
\begin{abstract}
We prove two universality theorems for neural networks consisting of a single neuron.  First we prove a strong form of the Bovdi--Geche universality conjecture on neurons over finite fields, where the activation function is a certain modification of the discrete logarithm, using the Weil bound on multiplicative character sums over finite fields. Then we show a generalisation of Aizenberg's universality theorem for complex binary neurons with a weight vector independent of the target function, positively answering another question in the literature. The proof of the latter uses an equidistribution theorem for the dilates of certain maps into tori.
\end{abstract}
\maketitle
\pagestyle{myheadings}
\markboth{Victor Bovdi and Ambrus P\'al}{Proofs of the Bovdi--Geche conjecture and a complex analogue}
\section{Introduction}

Neural networks consisting of a single neuron have been much studied in the literature, including a class of functions which are usually called threshold functions, both over finite fields (for example see \cite{Thelma}, \cite{Geche1998}, \cite[Chapter~4]{Geche2010}, and \cite{GecheEtAl2020}), and over the complex field (for example see \cite{AIP1971Doklady}, \cite{AIPK1971}, \cite{AIPK1973}, \cite{AizenbergIvaskiv1977}, \cite{Aizenberg2009} and \cite{Aizenberg2010ISMVL}). The universality of such functions on finite subsets is of fundamental importance, and much effort was exerted on this circle of problems. In this paper we prove two universality theorems for threshold functions. In order to state them, we will need to review some basic definitions. 
\begin{defn} Let $K$ be a field, let $Y$ be a set, and let
$\Phi:K\dashrightarrow Y$ be a partially defined map, which we will call the activation function, with domain of definition $\mathcal D\subset K$. As usual for every $n\in\mathbb N$ let $\langle\cdot,\cdot\rangle:K^n\times K^n\to K$ denote the usual scalar product:
$$\langle\underline x,\underline y\rangle=\sum_{i=1}^n x_i y_i\textrm{, where }
\underline x=(x_1,\ldots,x_n),\underline y=(y_1,\ldots,y_n)\in K^n.$$
For every $\underline w\in K^n$ and $T\in K$ the threshold function attached to $(\underline w,T)$ with respect to $\Phi$ is the partially defined function:
$$\theta:K^n\dashrightarrow Y$$
which maps every $\underline x\in K^n$ such that $\langle\underline x,\underline w\rangle+T\in\mathcal D$ to the value $\theta(\underline x)=
\Phi(\langle\underline x,\underline w\rangle+T)$.
\end{defn}
\begin{defn} A pair $(\underline w,T)\in K^n\times K$ is admissible with respect to a subset $X\subseteq K^n$ if the associated threshold function $\theta$ is defined on every element of $X$. The threshold function $\theta$ attached to $(\underline w,T)\in K^n\times K$ realises a function $f:X\to Y$ if the pair $(\underline w,T)$ is admissible with respect to $X$ and the restriction of $\theta$ to $X$ is $f$. In this case we will also sometimes say that $(\underline w,T)$ realises $f$ for the sake of simplicity. The pair $(K,\Phi)$ is universal for $X$ if every map $f:X\to Y$ is realised by some threshold function.
\end{defn}
\begin{defn} Fix integers
\begin{equation}\label{eq:para}
n\geq1,\qquad 2\leq k_1\leq\cdots\leq k_n,\qquad q\geq2,\end{equation}
set $N=k_1\cdots k_n$, and let $L=\operatorname{lcm}(k_1,\ldots,k_n,q)$.  A finite field $\mathbb F$ of cardinality $|\mathbb F|=Q$ is called admissible for the parameters in \eqref{eq:para} if
$$L\mid Q-1\quad\text{and}\quad Q-1>\max\{k_1,\ldots,k_n,q\}.$$
By Dirichlet's theorem on primes in arithmetic progressions there are admissible finite fields of arbitrarily large cardinality. 
\end{defn}
\begin{defn}\label{def1.4} For every divisor $d$ of $Q-1$ let $C_d$ denote the unique subgroup of $\mathbb F^*$ of order $d$.  Set
$$X=C_{k_1}\times\cdots\times C_{k_n}\subset(\mathbb F^*)^n\subset
\mathbb F^n, \qquad Y=C_q\subset\mathbb F^*.$$
Fix a primitive element $\varepsilon\in\mathbb F^*$, set $h=\frac{Q-1}{q}$ and $\sigma=\varepsilon^{h}$. The activation function attached to $(\mathbb F,\varepsilon)$ is the map
$$\Phi_{\varepsilon}:\mathbb F^*\longrightarrow Y$$
given by the rule: 
\begin{equation}\label{eq:activation}
 \Phi_{\varepsilon}(\varepsilon^a)=\sigma^{\lfloor a/h\rfloor},
 \qquad 0\leq a\leq Q-2.
\end{equation}
For every function $f:X\to Y$ let $R_f(\mathbb F,\varepsilon)$ be the number of pairs $(\underline w,T)\in \mathbb F^n\times \mathbb F$ admissible with respect to $X$ such that the threshold function attached to $(\underline w,T)$ with respect to $\Phi_{\varepsilon}$ realises $f$.
\end{defn}
\begin{thm}\label{thm:counting} There exists a constant $C_N>0$, depending only on $N$, such that for every admissible finite field $\mathbb F$ of cardinality $Q$, every map $f:X\to Y$, and every primitive element
$\varepsilon\in\mathbb F^*$ we have
\begin{equation}\label{eq:counting-asymptotic}
 \left|R_f(\mathbb F,\varepsilon)-\frac{Q^{n+1}}{q^N}\right|
 \leq C_N Q^{n+1/2}(1+\log Q)^N.
\end{equation}
\end{thm}
Since in the estimate above the main term has order $Q^{n+1}$ and the error is bounded by a constant multiple of $Q^{n+1/2}(1+\log Q)^N$, we immediately get the following 
\begin{cor} Let $\mathbb F$ be an admissible finite field of cardinality $Q$, and let $\varepsilon\in\mathbb F^*$ be a primitive element. If $Q$ is sufficiently large (for fixed $N$ and $q$) then the pair $(\mathbb F,\Phi_{\varepsilon})$ is universal for $X$.\qed 
\end{cor}
This claim is stronger than the conjecture formulated by Bovdi--Geche on page 125 of \cite{Geche2010}, which only postulated the existence of 
an admissible field $\mathbb F$ such that $(\mathbb F,\Phi_{\varepsilon})$ is universal for some choice of a primitive element $\varepsilon\in\mathbb F^*$. 

The proof of Theorem \ref{thm:counting} is a rather standard application of the Weil bound for multiplicative character sums; for similar arguments see for example \cite{DavenportErdos}, \cite{Peralta} and \cite{MSW}. As usual it proceeds by expressing the characteristic function of the set of pairs
$(\underline w,T)$ whose threshold function realises $f$ in terms of multiplicative characters. The way this proof mainly differs from the arguments in the aforementioned papers is that while they work with the partition of the multiplicative group of the field into the cosets of a subgroup of fixed index, we work with the level sets of the activation function
$\Phi_{\varepsilon}$ which are not cosets. However this difference is taken care of by the elementary Lemma~\ref{lem:l1-bound}. 
\begin{defn}\label{angular} Now fix an integer $q\geq2$, and set $K=\mathbb C$ and $Y=\mathbb Z/q\mathbb Z$. As usual let $\{\cdot\}:\mathbb R\to [0,1)$ denote the fractional-part function. For every integer
$\ell\geq1$ let $\Phi_{q,\ell}:\mathbb C\dashrightarrow Y$ denote the partially defined function given by the rule:
$$\Phi_{q,\ell}(z)=j\ \textrm{mod}\ q\quad\textrm{if and only if}\quad 
 \big\{\frac{\ell\cdot\arg(z)}{2\pi}\big\}\in \left(\frac{j}{q},\frac{j+1}{q}\right)$$
for every $j\in\{0,1,\ldots,q-1\}$, with domain of definition $\mathbb C^*\setminus\mathcal R_{q,\ell}$, where 
$$\mathcal R_{q,\ell}=
 \big\{z\in\mathbb C^*\mid \big\{\frac{\ell\cdot\arg(z)}{2\pi}\big\}=\frac{j}{q}
 \text{ for some }0\leq j<q\big\},$$
and $\arg:\mathbb C\dashrightarrow\mathbb R$ is a branch of the argument function on $\mathbb C^*\setminus\mathcal R_{q,\ell}$. 
\end{defn}

\begin{thm}\label{thm:strong-periodic-universality} Let $X\subset\mathbb C^n$ be nonempty and finite. Then there exist a $\underline w\in\mathbb C^n$ and an integer $\ell_0\geq1$ such that for every integer $\ell\geq\ell_0$ and every map $f:X\to Y$ there is a $b_{f,\ell}\in\mathbb C$ such that $(\underline w,b_{f,\ell})$ realises $f$ with respect to $\Phi_{q,\ell}$. 
\end{thm}

\begin{rem} The construction of complex threshold functions in this paper originates in the works \cite{AIP1971Doklady},\cite{AIPK1971},\cite{AIPK1973} and \cite{AizenbergIvaskiv1977}. Since it was known that these functions are not universal in general when $\ell=1$,  I.~N.~Aizenberg introduced the more general activation function in Definition \ref{angular} (see \cite[formula~(5), p.~350]{Aizenberg2009}), and  on p.~348 of the same paper he asked whether increasing $\ell$ makes arbitrary $q$-valued functions realisable. For $q=2$, that is for Boolean output, universality was already known for $X=\{-1,1\}^n$ (see \cite[pp.~467--473]{Aizenberg1991}), but it was open for general $q$ until our theorem. The latter is actually much stronger, since we can choose the weight vector $\underline w$ uniformly which will work for every function $f$ and every sufficiently large $\ell$.
\end{rem}

The proof of Theorem \ref{thm:strong-periodic-universality} uses an equidistribution theorem for the dilates of certain maps into tori. Such results are classical, see for example \cite{Randol} and its extension to nilmanifolds in \cite[Theorem~1.2]{KraShahSun}. The vector $\underline w$ is chosen so that the conditions of the equidistribution theorem apply. Our strategy is very similar to Yarotsky's neural-network construction in \cite[pp.~11934--11935]{Yarotsky2021}, although the latter uses Kronecker's density theorem instead. 
\begin{cont} In Section~2 we prove Theorem~\ref{thm:counting}, while in 
Section~3 we show Theorem~\ref{thm:strong-periodic-universality} using the methods we described above.
\end{cont}
\begin{ackn} The authors acknowledge that the initial sketches of the paper were developed via conversations with ChatGPT, but the authors thoroughly reviewed and revised the manuscript, and they take full responsibility for the final version of the paper. 
\end{ackn}

\section{The Bovdi--Geche conjecture and multiplicative character sums}

\begin{notn}  Let $\mathbb F$ be a finite field of cardinality $|\mathbb F|=Q$ which is admissible for the parameters in \eqref{eq:para} and let $\varepsilon\in\mathbb F^*$ be a primitive element as above. Write $\mathbb F^*$ as the disjoint union of the sets
\begin{equation}\label{eq:blocks}
 B_j=\{\varepsilon^{jh},\varepsilon^{jh+1},\ldots,
       \varepsilon^{(j+1)h-1}\},
 \qquad 0\leq j\leq q-1.
\end{equation}
Let $\widehat{\mathbb F}^*$ denote the group of complex-valued multiplicative characters of $\mathbb F^*$. Let $\chi_0\in\widehat{\mathbb F}^*$ be the principal character. We extend every character $\chi\in
\widehat{\mathbb F}^*$ to $\mathbb F$ by setting $\chi(0)=0$. For $0\leq j<q$, $z\in\mathbb F$ and $\chi\in\widehat{\mathbb F}^*$, define
\begin{equation}\label{eq:fourier-coefficient}
 c_j(\chi)=\frac{1}{Q-1}\sum_{z\in B_j}\overline{\chi(z)}.
\end{equation}
As usual, for every subset $B\subseteq\mathbb F$ let $ 1_B:\mathbb F
\to\mathbb C$ denote the characteristic function of $B$. 
\end{notn}
\begin{lemma} We have:
\begin{equation}\label{eq:fourier-inversion} 
 1_{B_j}(z)=\sum_{\chi\in\widehat{\mathbb F}^*}c_j(\chi)\chi(z)
 \qquad(\forall z\in\mathbb F).
\end{equation}
\end{lemma}
\begin{proof} The equality for $z=0$ follows from our convention for extending characters, while for other values it is immediate from the multiplicative Fourier inversion formula.
\end{proof}

\begin{lemma}\label{lem:l1-bound}
There is an absolute constant $C$ such that
$$c_j(\chi_0)=\frac1q,\quad\textrm{and}\quad
 \sum_{\chi\in\widehat{\mathbb F}^*}|c_j(\chi)|
 \leq C(1+\log Q)$$
for every choice of $q$, the field $\mathbb F$, the primitive element
$\varepsilon$ and $j$.
\end{lemma}

\begin{proof} Clearly $c_j(\chi_0)=\frac{h}{Q-1}=\frac{1}{q}$, so the first claim holds. For simplicity let $M$ denote $Q-1$, and let $\rho$ be the unique generator of $\widehat{\mathbb F}^*$ such that $\rho(\varepsilon)=e^{2\pi i/M}$. For $s\in\{1,2,\ldots,M-1\}$ using the definition of the set $B_j$ we get
$$\begin{aligned}
 |c_j(\rho^s)|
 &=\left|\frac{1}{M}\sum_{a=0}^{h-1}
      \overline{\rho^s(\varepsilon^{jh+a})}\right|\\
 &=\left|\frac{e^{-2\pi i sjh/M}}{M}\sum_{a=0}^{h-1}e^{-2\pi i sa/M}\right|\\
 &=\frac{1}{M}\left|\sum_{a=0}^{h-1}e^{-2\pi i sa/M}\right|\\
 &=\frac{1}{M}\left|\frac{1-e^{-2\pi i sh/M}}{1-e^{-2\pi i s/M}}\right|\\
 &=\frac{|\sin(\pi sh/M)|}{M|\sin(\pi s/M)|}\leq\frac{1}{M\sin(\pi s/M)},
\end{aligned}$$
where we used that the factor $e^{-2\pi i sjh/M}$ has absolute value one in the third equation, the geometric-sum formula in the fourth equation, the identity $|1-e^{it}|=2|\sin(t/2)|$ in the fifth equation, and the inequality
$|\sin(t)|\leq1$ in the last inequality. Therefore
$$\begin{aligned}
 \sum_{\chi\in\widehat{\mathbb F}^*}|c_j(\chi)|
 &=\sum_{s=0}^{M-1}|c_j(\rho^s)|\\
  &\leq\frac1q+\sum_{s=1}^{M-1}\frac{1}{M\sin(\pi s/M)}
  \end{aligned}$$
  $$\begin{aligned}
   &\leq\frac1q+2\sum_{s=1}^{\lfloor M/2\rfloor}
   \frac{1}{M\sin(\pi s/M)}\\
 &\leq\frac1q+2\sum_{s=1}^{\lfloor M/2\rfloor}
   \frac{1}{2s}\leq C(1+\log Q),
 \end{aligned}$$
where we used
$$\sin\left(\frac{\pi(M-s)}{M}\right)
 =\sin\left(\frac{\pi s}{M}\right)$$
in the second inequality, the elementary inequality:
 $$\sin u\geq\frac{2u}{\pi}\qquad(0\leq u\leq\pi/2),$$
which follows from the concavity of sine on this interval, in the third inequality, 
and the usual estimate for the harmonic series in the last inequality. 
\end{proof}
\begin{lemma}\label{lem:affine-character-sum} Let $Z$ be a set of $N$ distinct points of $\mathbb F^n$.  For each $\underline x\in Z$ let
$\chi_{\underline x}\in\widehat{\mathbb F}^*$, and assume that not all
$\chi_{\underline x}$ are principal. Then
\begin{equation}\label{eq:affine-character-bound}
 \left|\sum_{\underline w\in\mathbb F^n}\sum_{T\in\mathbb F}
 \prod_{\underline x\in Z}\chi_{\underline x}(T+\langle \underline x,\underline w\rangle)
 \right| \leq c_N Q^{n+1/2},\end{equation}
where $c_N$ depends only on $N$.
\end{lemma}

\begin{proof} Let $\mathcal A\subseteq\mathbb F^n$ denote the set:
$$\mathcal A=\{\underline w\in\mathbb F^n\mid\exists\underline x,\underline y\in Z\textrm{ such that }\underline x\neq\underline y\textrm{ and }\langle\underline x,\underline w\rangle=\langle \underline y,\underline w\rangle\},$$
and let $\mathcal B\subseteq\mathbb F^n$ be its complement. Clearly $\mathcal A$ is the union of the hyperplanes
$$\{\underline w\in\mathbb F^n\mid\langle\underline x-\underline y,\underline w\rangle=0\},$$
where $\underline x,\underline y\in Z$ are different, so its cardinality is at most $\binom{N}{2}Q^{n-1}$. Therefore
\begin{equation}\label{2.4.2}
\left|\sum_{\underline w\in\mathcal A}\sum_{T\in\mathbb F}
 \prod_{\underline x\in Z}\chi_{\underline x}
 (T+\langle \underline x,\underline w\rangle)
 \right| \leq \binom{N}{2}Q^n,
 \end{equation}
 since each product appearing in the sum has absolute value at most $1$. 

Let $\rho$ generate the character group $\widehat{\mathbb F}^*$ as above, and for every $\underline x\in Z$ let $e(\underline x)\in\{0,1,\ldots,Q-2\}$ be the unique number such that $\chi_{\underline x}=\rho^{e(\underline x)}$. Now let $\underline w\in\mathcal B$ be arbitrary. By assumption at least one of the $e(\underline x)$ is nonzero, so the polynomial
$$P_{\underline w}(t)=\prod_{\substack{\underline x\in Z\\e(\underline x)\neq0}}
   (t+\langle \underline x,\underline w\rangle)^{e(\underline x)}$$
is well-defined and has positive degree. Moreover the elements
$\langle \underline w,\underline x\rangle$ are pairwise distinct as
$\underline x$ ranges over $Z$, so every root of $P_{\underline w}(t)$ has multiplicity at most $Q-2$, and hence this polynomial is not a constant multiple of the $(Q-1)$-th power of a polynomial. As the order of $\rho$ is $Q-1$, we may apply the standard Weil bound for the multiplicative character $\rho$ evaluated on the polynomial $P_{\underline w}(t)$ (see Theorem~5.41 of \cite{LN} on page 225) to get that
$$\left|\sum_{T\in\mathbb F}\rho(P_{\underline w}(T))\right|
 \leq(N-1)Q^{1/2}.$$
 Note that
 $$\rho(P_{\underline w}(T))= \prod_{\underline x\in Z}\chi_{\underline x}(T+\langle \underline x,\underline w\rangle)$$
 for every $T\in\mathbb F$ outside the set
$$E_{\underline w}= \{-\langle\underline x,\underline w\rangle\mid\underline x\in Z\}.$$
There are at most $N$ elements of $E_{\underline w}$, and at each such value the two expressions differ by a quantity of absolute value at most $1$, so we get that
 \begin{equation}\label{2.4.3}
 \left|\sum_{\underline w\in\mathcal B}\sum_{T\in\mathbb F}
 \prod_{\underline x\in Z}\chi_{\underline x}
 (T+\langle \underline x,\underline w\rangle)
 \right| \leq Q^n\big((N-1)Q^{1/2}+N\big).
 \end{equation}
Combining the inequalities (\ref{2.4.2}) and (\ref{2.4.3}) together we get the claim. 
\end{proof}
\begin{notn} For every function $f:X\to Y$ and for every $\underline x\in X$ let $j(f,\underline x)\in\{0,1,\ldots,q-1\}$ be the unique number such that 
$$f(\underline x)=\sigma^{j(f,\underline x)}.$$
\end{notn}
\begin{lemma}\label{lem:counting-identity} We have:
\begin{equation}\label{eq:realization-count}
R_f(\mathbb F,\varepsilon)=\sum_{\underline w\in\mathbb F^n}\sum_{T\in\mathbb F}
 \prod_{\underline x\in X}
 1_{B_{j(f,\underline x)}}(T+\langle \underline x,\underline w\rangle).
\end{equation}
\end{lemma}

\begin{proof} For every $(\underline w,T)\in\mathbb F^{n+1}$ and for every $\underline x\in X$ we have
$$1_{B_{j(f,\underline x)}}(T+\langle \underline x,\underline w\rangle)=1$$
if and only if
$$T+\langle \underline x,\underline w\rangle\neq0\quad\textrm{and}
\quad\Phi_{\varepsilon}(T+\langle \underline x,\underline w\rangle)=\sigma^{j(f,\underline x)}=f(\underline x).$$
Therefore
$$\prod_{\underline x\in X}1_{B_{j(f,\underline x)}}
(T+\langle \underline x,\underline w\rangle)$$
is equal to $1$ if and only if $(\underline w,T)$ is admissible with respect to $X$ and realises $f$, and it is zero otherwise.  Summing this product
over all pairs $(\underline w,T)$ gives \eqref{eq:realization-count}.
\end{proof}

\begin{proof}[Proof of Theorem \ref{thm:counting}] By plugging equation (\ref{eq:fourier-inversion}) into the identity in Lemma \ref{lem:counting-identity} we get
\begin{equation}\label{eq:character-expanded-count}
 \begin{aligned}
 R_f(\mathbb F,\varepsilon)
 &=\sum_{\underline w\in\mathbb F^n}\sum_{T\in\mathbb F}
   \prod_{\underline x\in X}
   \left(\sum_{\chi_{\underline x}\in\widehat{\mathbb F}^*}
   c_{j(f,\underline x)}(\chi_{\underline x})
   \chi_{\underline x}(T+\langle\underline x,\underline w\rangle)\right)\\
 &=\sum_{\underline w\in\mathbb F^n}\sum_{T\in\mathbb F}
   \sum_{(\chi_{\underline x})_{\underline x\in X}\in
   (\widehat{\mathbb F}^*)^X}
   \left(\prod_{\underline x\in X}c_{j(f,\underline x)}(\chi_{\underline x})\right)\\
 &\hspace{23mm}\cdot
   \prod_{\underline x\in X}
   \chi_{\underline x}(T+\langle\underline x,\underline w\rangle)\\
 &=\sum_{(\chi_{\underline x})_{\underline x\in X}\in
   (\widehat{\mathbb F}^*)^X}
   \left(\prod_{\underline x\in X}c_{j(f,\underline x)}(\chi_{\underline x})\right)\\
 &\hspace{23mm}\cdot
   \sum_{\underline w\in\mathbb F^n}\sum_{T\in\mathbb F}
   \prod_{\underline x\in X}
   \chi_{\underline x}(T+\langle\underline x,\underline w\rangle),
 \end{aligned}
\end{equation}
where we have expanded the finite product of finite sums in the second equality, and we have exchanged the finite sums over $(\underline w,T)$ and over the tuples $(\chi_{\underline x})_{\underline x\in X}\in
(\widehat{\mathbb F}^*)^X$ in the third equality. 

Consider first the tuple $(\chi_{\underline x})_{\underline x\in X}\in
(\widehat{\mathbb F}^*)^X$ in which every $\chi_{\underline x}$ is the principal character $\chi_0$.  By Lemma \ref{lem:l1-bound} we have
$c_{j(f,\underline x)}(\chi_0)=1/q$ for every $\underline x$, therefore the corresponding contribution in the right hand side of (\ref{eq:character-expanded-count}) is independent of $f$ and equals
$$\frac{1}{q^N}\#\big\{(\underline w,T)\mid
 T+\langle \underline x,\underline w\rangle\neq0
 \text{ for every }\underline x\in X\big\}.$$
Each equation $T+\langle \underline w,\underline x\rangle=0$ defines a
hyperplane of cardinality $Q^n$ in $\mathbb F^{n+1}$, so 
$$|\big(\prod_{\underline x\in X}c_{j(f,\underline x)}(\chi_0)\big)
 \cdot\sum_{\underline w\in\mathbb F^n}\sum_{T\in\mathbb F}
   \prod_{\underline x\in X}
   \chi_0(T+\langle\underline x,\underline w\rangle)-\frac{Q^{n+1}}{q^N}|
   \leq C'_NQ^n$$
for some constant $C_N'$ depending only on $N$.  All remaining character tuples satisfy the conditions of Lemma \ref{lem:affine-character-sum}, so the latter combined with Lemma \ref{lem:l1-bound} bounds the absolute value of their total contribution to the right hand side of (\ref{eq:character-expanded-count}) by
$$c_N Q^{n+1/2}
 \prod_{\underline x\in X}
 \left(\sum_{\chi\in\widehat{\mathbb F}^*}
 |c_{j(f,\underline x)}(\chi)|\right)
 \leq C''_NQ^{n+1/2}(1+\log Q)^N$$
for some constant $C_N''$ depending only on $N$. Now \eqref{eq:counting-asymptotic} is clear. 
\end{proof}

\section{The universality of periodic complex neurons and equidistribution of maps}
\label{sec:periodic-universality}

We will start with the following easy
\begin{lemma}\label{lem:separating-complex-weight} Let $Z\subset\mathbb C^n$ be finite. Then there is a vector $\underline w\in\mathbb C^n$ such that the numbers $\langle\underline w,\underline x\rangle$ are pairwise distinct as
$\underline x$ ranges over $Z$. 
\end{lemma}

\begin{proof} For every $\underline x,\underline y\in Z$ such that $\underline x\neq\underline y$ the equation $\langle\underline x,\underline w\rangle=\langle \underline y,\underline w\rangle$ defines a proper complex hyperplane in
$\mathbb C^n$.  A finite union of proper hyperplanes cannot cover
$\mathbb C^n$, so there is a $\underline w$ outside their union with the required property.
\end{proof}

\begin{notn} Now we are going to use the notation of Definition \ref{angular} and Theorem \ref{thm:strong-periodic-universality}, and fix a vector
$\underline w$ as in Lemma \ref{lem:separating-complex-weight} above for $Z=X\subset\mathbb C^n$. Write $a_{\underline x}=
\langle\underline x,\underline w\rangle$ for every $\underline x\in X$ and
choose an open disk $U\subset\mathbb C$ whose closure does not intersect the finite set $\{-a_{\underline x}\mid\underline x\in X\}$. On $U$ for each $\underline x\in X$ choose a real-analytic branch $\theta_{\underline x}$ of $b\mapsto\frac{1}{2\pi}\arg(b+a_{\underline x})$, and set
\begin{equation}\label{eq:phase-map}
 \Theta:U\longrightarrow\mathbb R^X,
 \qquad
 \Theta(b)=(\theta_{\underline x}(b))_{\underline x\in X}.
\end{equation}
For every $\underline h=(h_{\underline x})_{\underline x\in X}\in\mathbb Z^X$ let $\phi_{\underline h}:U\to\mathbb R$ denote the real-analytic
function given by the rule:
$$b\longmapsto\langle\underline h,\Theta(b)\rangle
=\sum_{\underline x\in X}h_{\underline x}\theta_{\underline x}(b).$$
\end{notn}

\begin{lemma}\label{lem:phase-nonresonance} For every nonzero
$\underline h\in\mathbb Z^X$ the function $\phi_{\underline h}$ is not constant on $U$.
\end{lemma}

\begin{proof} First note that the function
$$R_{\underline h}(b)=
\prod_{\underline x\in X}(b+a_{\underline x})^{h_{\underline x}}$$
is well-defined and nonzero everywhere on $U$, since $-a_{\underline x}\notin U$ for every $\underline x\in X$. Now assume that
$\phi_{\underline h}$ is constant. Then $R_{\underline h}$ has constant argument on $U$, so by the open mapping theorem $R_{\underline h}$ is constant. Therefore its logarithmic derivative:
\begin{equation}\label{eq:logarithmic-derivative}
 \frac{R_{\underline h}'(b)}{R_{\underline h}(b)}
 =\sum_{\underline x\in X}\frac{h_{\underline x}}{b+a_{\underline x}}
\end{equation}
vanishes on $U$. This is a rational function, so it vanishes identically by the unique continuation principle. Since the $a_{\underline x}$ are distinct, the residue of this logarithmic derivative at $-a_{\underline x}$ is $h_{\underline x}$.  Hence $h_{\underline x}=0$ for every $\underline x\in X$ which is a contradiction.
\end{proof}

\begin{defn} Let $\mathbb T^X$ denote the quotient torus $\mathbb R^X/\mathbb Z^X$, and let $\mathrm{mod}:\mathbb R^X\to\mathbb T^X$ denote the quotient map. Let $\mu$ denote the usual planar measure (or area) on subsets of $\mathbb C$, and let $\psi:U\to\mathbb R$ be a smooth, compactly supported and nonnegative function such that
$\int_U\psi\,d\mu=1$.  For every positive integer $\ell$ let $\mu_\ell$ denote the measure on $\mathbb T^X$ which for every continuous $F:\mathbb T^X\to\mathbb C$ is given by the formula:
$$\int_{\mathbb T^X}F(\underline y)\,d\mu_\ell(\underline y)
 =\int_U F\bigl(\mathrm{mod}(\ell\Theta(b))\bigr)\psi(b)\,d\mu(b).$$
\end{defn}

\begin{lemma}\label{lem:phase-equidistribution} The measures $\mu_\ell$ converge weakly to the Haar probability measure on $\mathbb T^X$ as
$\ell\to\infty$.
\end{lemma}

\begin{proof} For every $\underline h\in\mathbb Z^X$ let
$\widehat\mu_\ell(\underline h)$ denote the integral:
$$\int_{\mathbb T^X}e^{2\pi\mathbf i\langle\underline h,\underline y\rangle}\,d\mu_\ell(\underline y)
 =\int_U e^{2\pi\mathbf i\ell\langle\underline h,\Theta(b)\rangle}\psi(b)\,d\mu(b).$$
When $\underline h=0$ we get that
$$\widehat\mu_\ell(\underline 0)=\int_U\psi\,d\mu=1$$
by our assumptions on $\psi$. Hence it will be sufficient to show that
$\widehat\mu_\ell(\underline h)\to0$ as $\ell\to\infty$ for every non-zero
$\underline h$ by the Fourier criterion for weak convergence of measures on a torus. 

By Lemma~\ref{lem:phase-nonresonance} the real-analytic function $\phi_{\underline h}$ is not constant on the disk $U$.  Its critical set
$$C_{\underline h}=\{b\in U\mid \nabla\phi_{\underline h}(b)=0\}$$
has planar measure zero.  Indeed, if both first partial derivatives of
$\phi_{\underline h}$ vanished identically, then $\phi_{\underline h}$ would be constant because $U$ is connected.  Hence at least one of these partial derivatives is a nonzero real-analytic function. The zero set of a nonzero real-analytic function has planar measure zero, and $C_{\underline h}$ is contained in that zero set.

Let $K=\operatorname{supp}(\psi)$, which is a compact subset of $U$, and let $\varepsilon>0$.  Since $C_{\underline h}$ has measure zero, there is an open set $V$ whose closure lies in $U$ and which contains $C_{\underline h}\cap K$ such that
$$\int_V\psi(b)\,d\mu(b)<\varepsilon.$$
Choose a smooth, compactly supported function $\eta:U\to\mathbb R$ with $0\leq\eta\leq1$ such that $\eta$ vanishes on a neighborhood of $C_{\underline h}\cap K$ and $\eta=1$ on $K\setminus V$. Then
$$\left|\int_U e^{2\pi i\ell\phi_{\underline h}(b)}
 (1-\eta(b))\psi(b)\,d\mu(b)\right|
 \leq \int_V\psi(b)\,d\mu(b)<\varepsilon.$$
Set $a=\eta\psi$, and for this calculation write
$\phi=\phi_{\underline h}$ and $\lambda=2\pi\ell$.
Since the compact set $\operatorname{supp}(a)$ is disjoint from
$C_{\underline h}$, the function $|\nabla\phi|$ is bounded below by a
positive constant on a neighborhood of this set. On that neighborhood define
$$W=W_{\underline h}=\frac{\nabla\phi}{|\nabla\phi|^2}.$$
Here gradients and divergences are taken with respect to the two real
coordinates of complex number and we will let $\boldsymbol\cdot$ denote the corresponding real scalar product. The vector field $aW$ extends by zero
to a smooth, compactly supported vector field on $U$.

Choose an open disk $D$ containing $\operatorname{supp}(a)$ such that
$\overline D\subset U$, and let $\mathbf n$ denote the outward unit
normal on $\partial D$. Green's theorem in divergence form, applied
separately to the real and imaginary parts, gives
\begin{align*}
 \int_D\operatorname{div}\bigl(e^{i\lambda\phi}aW\bigr)\,d\mu
 &=\int_{\partial D}e^{i\lambda\phi}aW
   \boldsymbol\cdot\mathbf n\,ds=0,
\end{align*}
because $aW$ vanishes near $\partial D$. The product rule and the identity
$W\boldsymbol\cdot\nabla\phi=1$ on a neighborhood of
$\operatorname{supp}(a)$ imply that
\begin{align*}
 \operatorname{div}\bigl(e^{i\lambda\phi}aW\bigr)
 &=\nabla\bigl(e^{i\lambda\phi}\bigr)\boldsymbol\cdot(aW)
   +e^{i\lambda\phi}\operatorname{div}(aW)\\
 &=i\lambda e^{i\lambda\phi}a
   +e^{i\lambda\phi}\operatorname{div}(aW).
\end{align*}
Integrating this identity and rearranging therefore gives
\[
 \int_U e^{i\lambda\phi}a\,d\mu
 =-\frac{1}{i\lambda}
   \int_U e^{i\lambda\phi}\operatorname{div}(aW)\,d\mu.
\]
Here the integrals over $D$ may be replaced by integrals over $U$ because
the integrands are supported in $\operatorname{supp}(a)\subset D$.
Taking absolute values, applying the triangle inequality for integrals,
and using $|e^{i\lambda\phi}|=1$, we obtain
\begin{align*}
 \left|\int_U e^{2\pi i\ell\phi_{\underline h}}
   \eta\psi\,d\mu\right|
 &=\frac{1}{2\pi\ell}
   \left|\int_U e^{i\lambda\phi}\operatorname{div}(aW)\,d\mu\right|\\
 &\leq\frac{1}{2\pi\ell}
   \int_U\left|\operatorname{div}(aW)\right|\,d\mu\\
 &=\frac{C_{\underline h,V,\eta,\psi}}{\ell},
\end{align*}
where explicitly
\[
 C_{\underline h,V,\eta,\psi}
 =\frac{1}{2\pi}\int_U
   \left|\operatorname{div}(\eta\psi W_{\underline h})\right|\,d\mu.
\]
This constant is finite because $aW$ is smooth and compactly supported,
and it is independent of $\ell$ once $\underline h,V,\eta,\psi$ are fixed.
Combining these two estimates we get that
$$\limsup_{\ell\to\infty}|\widehat\mu_\ell(\underline h)|\leq\varepsilon.$$
Since $\varepsilon>0$ was arbitrary we get that $\widehat\mu_\ell(\underline h)\to0$. 
\end{proof}

\begin{proof}[Proof of Theorem~\ref{thm:strong-periodic-universality}] We continue to use the notation introduced in this section so far. By slight abuse of notation for every congruence class $a\mod q$ let the same symbol $a$ denote its unique representative in $\{0,1,\ldots,q-1\}$. For every function $f:X\to Y$ let $\mathcal B_f\subset\mathbb T^X$ denote the image of the open box
$$\prod_{\underline x\in X}\left(\frac{f(\underline x)}{q},
\frac{f(\underline x)+1}{q}\right)
 \subset\mathbb R^X$$
under the quotient map $\mathrm{mod}:\mathbb R^X\to\mathbb T^X$.  The set $\mathcal B_f$ has Haar measure $q^{-|X|}$ and its boundary has Haar measure zero, so by Lemma~\ref{lem:phase-equidistribution} we get that
$$\mu_\ell(\mathcal B_f)\longrightarrow q^{-|X|}.$$
In particular $\mu_\ell(\mathcal B_f)>0$ for every sufficiently large
$\ell$. Since there are only finitely many functions from $X$ to $Y$, we get that there is an $\ell_0$ such that
$$\mu_\ell(\mathcal B_f)>0\qquad(\forall\ell\geq\ell_0)$$
for every $f$ as above.

Now let $\ell\geq\ell_0$ and $f$ be arbitrary. By the above there is a $b\in U$ such that $\mathrm{mod}(\ell\Theta(b))\in\mathcal B_f$.  Since $b\ne-a_{\underline x}$ we get that $b+a_{\underline x}$ is nonzero for every
$\underline x\in X$.  Moreover $b+a_{\underline x}\notin\mathcal R_{q,\ell}$ and 
$$\Phi_{q,\ell}(b+a_{\underline x})=f(\underline x)\qquad
(\forall\underline x\in X).$$
Since $a_{\underline x}=\langle\underline w,\underline x\rangle$ we get that
$(\underline w,b)$ realises $f$. 
\end{proof}

\end{document}